\pdfoutput=1

\documentclass[11pt]{article}

\usepackage[]{acl}

\usepackage{times}
\usepackage{latexsym}
\usepackage{amsthm}

\usepackage{algorithm}
\usepackage{algpseudocode}

\usepackage[T1]{fontenc}

\usepackage[utf8]{inputenc}

\usepackage{microtype}

\usepackage{inconsolata}
\usepackage{amsmath}
\usepackage{graphicx}
\usepackage{booktabs}
\usepackage{enumitem}
\usepackage{multirow}
\usepackage{xspace}

\newcommand{\ours}{\textsc{MinPrompt}\xspace}

\usepackage{balance}

\newtheorem{theorem}{Theorem}
\newtheorem{assumption}{Assumption}

\title{\ours: Graph-based Minimal Prompt Data Augmentation\\ for Few-shot Question Answering}

\author{
        \textbf{Xiusi Chen}$^1$ \ \ 
        \textbf{Jyun-Yu Jiang}$^2$ \ \ 
        \textbf{Wei-Cheng Chang}$^2$ \ \ \\
        \textbf{Cho-Jui Hsieh}$^1$ \ \ 
        \textbf{Hsiang-Fu Yu}$^2$ \ \ 
        \textbf{Wei Wang}$^1$
       \\ 
  University of California, Los Angeles$^1$ \ \ \ \ \ Amazon Search$^2$
  \\
  {\tt \{xchen,chohsieh,weiwang\}@cs.ucla.edu} \\ 
  {\tt \{jyunyu.jiang,weicheng.cmu,rofu.yu\}@gmail.com}
}

\begin{document}
\maketitle
\begin{abstract}
Recent advances in few-shot question answering (QA) mostly rely on the power of pre-trained large language models (LLMs) and fine-tuning in specific settings. Although the pre-training stage has already equipped LLMs with powerful reasoning capabilities, LLMs still need to be fine-tuned to adapt to specific domains to achieve the best results. In this paper, we propose to select the most informative data for fine-tuning, thereby improving the efficiency of the fine-tuning process with comparative or even better accuracy on the open-domain QA task.
We present \ours, a minimal data augmentation framework for open-domain QA based on an approximate graph algorithm and unsupervised question generation. We transform the raw text into a graph structure to build connections between different factual sentences, then apply graph algorithms to identify the minimal set of sentences needed to cover the most information in the raw text. We then generate QA pairs based on the identified sentence subset and train the model on the selected sentences to obtain the final model.
Empirical results on several benchmark datasets and theoretical analysis show that \ours is able to achieve comparable or better results than baselines with a high degree of efficiency, bringing consistent improvements in F-1 scores. 
\end{abstract}

\section{Introduction}

Question answering (QA) provides accurate responses to a series of questions based on given narrative contexts. Its diverse applications extend to areas such as chatbots~\cite{yang2019end}, dialogue systems~\cite{burtsev2018deeppavlov}, and instant information retrieval~\cite{esteva2021covid}, making it a key pursuit in the field of natural language processing (NLP). Supervised learning has traditionally been the approach for developing efficient QA systems that deliver commendable results~\cite{chen2024iteralign,tian2024tinyllm}. However, this method is intrinsically restricted by its reliance on a large set of annotated QA training examples, which becomes problematic due to the substantial cost associated with acquiring expert-level annotations.

Our research focuses on the few-shot QA task, an effort to address the QA challenge with the presence of only a limited number of training examples. The prevalent approaches under the few-shot setting either introduce a new task and pre-train an extensive language model from scratch~\cite{ram2021few}, or they fine-tune an already pre-trained model on the given training examples~\cite{chada2021fewshotqa,tian2024graph}. The fine-tuning stage is crucial in the sense that it stimulates the power of the LLMs obtained during the pre-training stage and makes the model align with the input/output distribution of a certain domain or dataset. However, with an increasing data size for fine-tuning, the training duration increases accordingly, which is undesirable, especially when the model size is also large~\cite{openai2023gpt}. As such, the importance of minimal data augmentation cannot be understated. The fine-tuning data, often a limited resource in our consideration (up to 128 shots), is directly used to adjust the parameters of a pre-trained model to enhance performance on the downstream task. The data is usually labeled by domain experts and thus could be time-consuming to obtain in large quantities. On the other hand, augmented data represents a broader dataset, generated in an unsupervised manner by converting statements into question-answer pairs. In QA tasks, it is vital for a model to be exposed to a diverse range of questions, answers, and contexts to develop a robust understanding of the language and the task at hand. However, not all parts of the training data hold equal relevance or significance for the model's learning process. Some parts may contain more valuable information or more complex language structures that the model needs to understand to improve its performance. Consequently, identifying and augmenting these critical portions of the training data could substantially enhance the model's capacity to answer questions accurately and comprehensively.

To address the above challenges, we present \ours, which consists of the following three modules: (1) A \textbf{sentence graph construction module} that leverages sentence graph representation to structurize the raw text. Each node in the graph symbolizes a sentence, while edges illustrate the shared entities between sentences. This sentence graph effectively encapsulates the complex interconnections between various textual elements; (2) A \textbf{data selection module} that features an approximate minimal dominating set algorithm. The algorithm is applied to the sentence graph to identify the smallest set of sentences to cover all shared entities. This module ensures efficient use of computational resources, reduces the risk of overfitting, and enhances the model's generalization ability, resulting in an overall improvement in QA performance; and (3) A \textbf{question generation module} that transforms the selected plain factual sentences into QA pairs. The synthesized QA pairs are further turned into prompts, providing a condensed, yet comprehensive representation of the text. The generated prompts serve as high-quality, information-rich training instances for the QA model. This model trained on the compact and meaningful prompts is then capable of generating accurate answers to the posed questions, all without requiring any additional explicit supervision.

In summary, our contributions are as follows:
\begin{itemize}[leftmargin=*]
    \item We propose to study minimal data augmentation for effective and efficient few-shot QA
    \item We introduce \ours, a minimal data augmentation framework that uses a graph-based algorithm and unsupervised question generation to synthesize the most informative QA training samples out of the raw text.
    \item We conduct extensive experiments on publicly accessible benchmarks to validate the effectiveness of \ours, and observe a solid improvement over competitive compared methods. Beyond that, we also study the necessity of different parts of the model.
\end{itemize}

\section{Related Work}
\noindent \textbf{Question generation.} 
\citet{chen2019reinforcement} presented an answer-aware question generation (QG) model that employs reinforcement learning for improved question quality. The model incorporates a coverage mechanism to alleviate the common issue of answer-related content being left out from the generated questions.
\citet{ma2020improving} developed a more sophisticated approach to answer-aware question generation. Their model uses sentence-level semantic matching and answer position inferring within a sequence-to-sequence framework, resulting in higher-quality questions.
\citet{do2023modeling} proposed a two-stage framework for Conversational Question Generation (CQG). It selects sentences from a semantic graph to pick up coherent topics and then uses a classifier to determine the answer type of the question. Their approach produces more natural dialogues, as real-life interlocutors often discuss relevant content that is non-sequential. 
\citet{mohammadshahi2022rquge} introduces RQUGE, a novel metric for assessing the quality of automatically generated questions. Traditional methods may unfairly penalize valid questions that don't mirror reference questions closely. RQUGE overcomes these issues by evaluating on the basis of the answerability of a question given the context. Utilizing pre-trained models for its QA scorer modules, RQUGE does not require additional training. The paper presents evidence of RQUGE's high correlation with human judgment and robustness against adversarial corruption.


\begin{figure*}[!t]
    \centering
    \includegraphics[width=\linewidth]{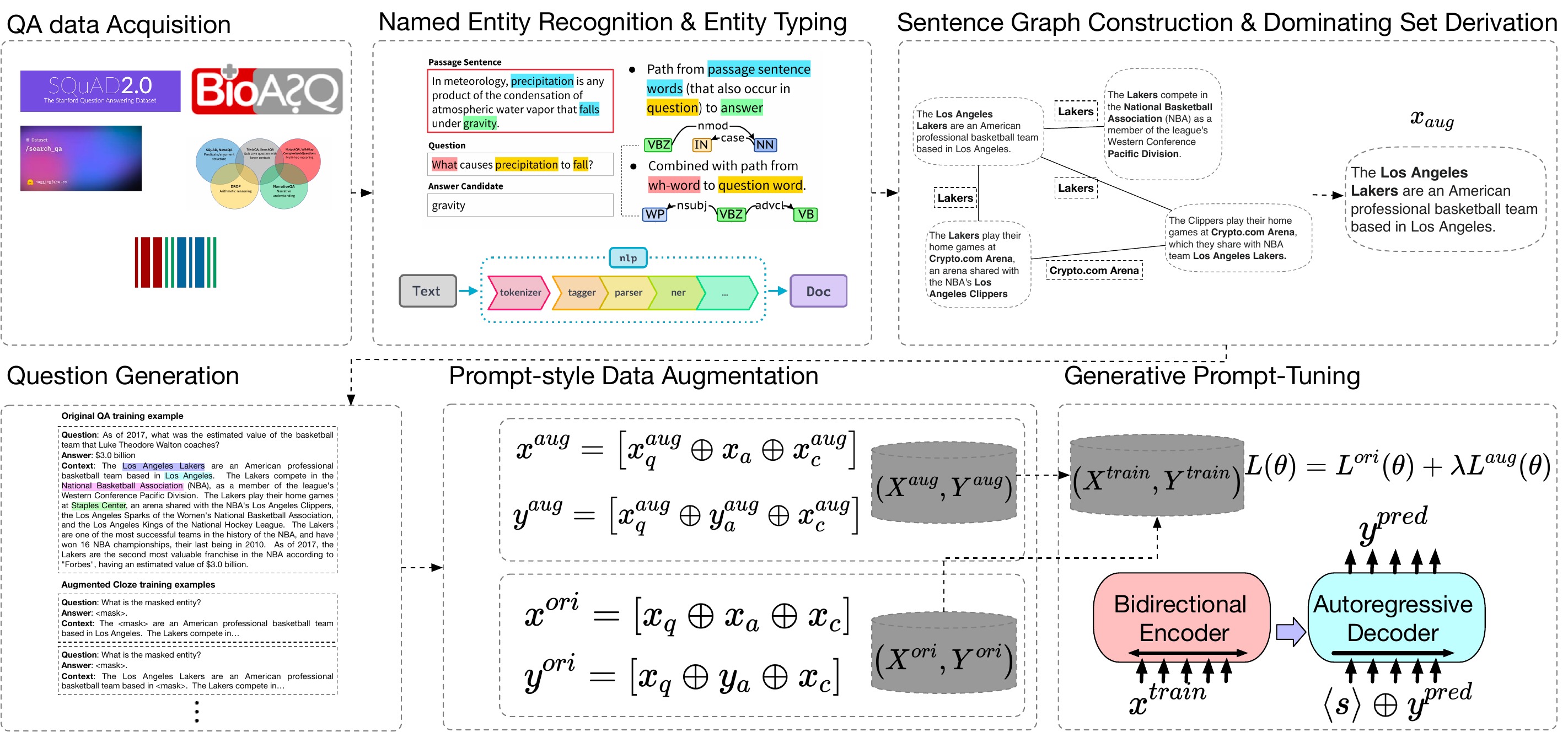}
    \caption{\textbf{Framework overview for \ours.}}
    \label{fig:overview}
\end{figure*}
\noindent \textbf{Few-shot QA.} Previous research in QA has mainly focused on either reusing pre-trained language models (PLMs) ~\cite{Lan2020ALBERT:,joshi2020spanbert} or training a model from scratch using synthetic QA data~\cite{puri2020training,lewis2019unsupervised,alberti2019synthetic}. However, both approaches require a large amount of annotated data from the downstream QA task to fine-tune the models, which can be impractical in real-world scenarios. To address this problem, several recent approaches have been developed that allow the model to adapt to the downstream task with only a small amount of annotated data~\cite{ram2021few,chada2021fewshotqa}. For example, \citet{ram2021few} proposed a pretraining scheme tailored for QA tasks by designing a recurring span selection objective that aligns with the common objective in extractive QA tasks. \citet{chada2021fewshotqa} proposed a framework called FewshotQA, which leverages the capacity of existing PLMs by constructing a QA-style prompt that casts the QA problem as a text generation problem, specifically by concatenating the question and a mask token representing the answer span. This approach aims to save pretraining the model on a large-scale corpus.
In contrast to these previous studies, this paper proposes to focus on identifying and leveraging more relevant information from the context data in addition to the annotated QA pairs to fine-tune the model in a few-shot setting.

\section{\ours: Graph-based Prompt Data Augmentation for Few-shot QA}
As shown in Figure~\ref{fig:overview}, our overall framework, \ours, is designed to extract the most semantically rich and factually dense sentences to serve as candidates for conversion into a prompt tuning QA dataset. This process is guided by the principal intuition that the most informative sentences are those that encompass facts or declarations concerning a greater number of entities. Hence, these high-impact sentences should ideally cite more entities within their purview.
To implement this, we start by extracting the co-reference of entities across sentences. 
Essentially, it allows us to map the discourse in a way that allows us to understand which sentences are speaking about the same entities. 
Next, we construct a graph to depict the higher-order coreference relationships. In this graph, the sentences serve as nodes, and sentences are connected if they mention the same entity. This representation allows us to establish and understand the intricate network of relationships between sentences and the entities they mention. 
Employing graph-based algorithms, we are then able to identify and extract the most informative sentences. These are typically sentences that have a high degree of connectivity in the graph, indicating that they mention or discuss a larger number of entities.
We then transform these selected sentences into a fine-tuning dataset. The transformation process entails restructuring the sentences to meet the format requirements of a QA dataset, which generally involves turning declarative sentences into question-and-answer pairs. 
This method thus combines insights from computational linguistics and graph theory to achieve its goal of creating a high-quality fine-tuning dataset for QA tasks. The approach ensures that the dataset is not only rich in informative sentences, but also maps intricate entity relationships, thus providing a comprehensive context for each question and answer pair. This context helps in the training of more robust and nuanced QA systems.

\subsection{Named Entity Recognition \& Entity Typing}
We use the entities as the bridge to build connections between all the factual sentences. We first conduct named entity recognition (NER) on the raw text to extract all the entity mentions along with their types. 
For the purpose of unsupervised QA data generation in our setting, the key lies in generating the questions given the raw text and the extracted entities (as answers). The most straightforward way to generate questions is to convert factual sentences into cloze questions~\cite{chen2023gotta}. Creating a conventional cloze question involves extracting the original sentence containing the answer from the context and replacing the answer with a chosen token. However, training a model on these data primarily imparts text-matching and fill-in-the-blank skills, while offering minimal generalizability. As a result, we opt for a retrieval-based method to procure a sentence akin to the one containing the answer and subsequently use this to formulate a question. This has been evidenced in the work by \cite{lewis-etal-2019-unsupervised} and further affirmed by our preliminary experiments. Our initial step involves indexing all sentences from a Wikipedia dump using the ElasticSearch search engine. Named entities were extracted from each sentence within the Wikipedia corpus as well as from the sentences utilized as queries. We presupposed access to a named-entity recognition system and leveraged the spaCy\footnote{\url{https://spacy.io}} NER pipeline for this work, which is proven effective in NER and entity typing. Subsequently, for a given context-answer pair, we queried the index. This query involved using the original context sentence to return a sentence that either (1) includes the answer, or (2) does not originate from the \emph{context}, thus discarding sentences with high similarity. Aside from guaranteeing that the retrieved sentence and the query sentence share the answer entity, we require that at least one additional matching entity be present in both the query sentence and the entire context. Finally, these retrieved sentences were introduced into our sentence graph construction module.

\subsection{Sentence Graph Construction}
\begin{figure}
    \centering
    \includegraphics[width=\linewidth]{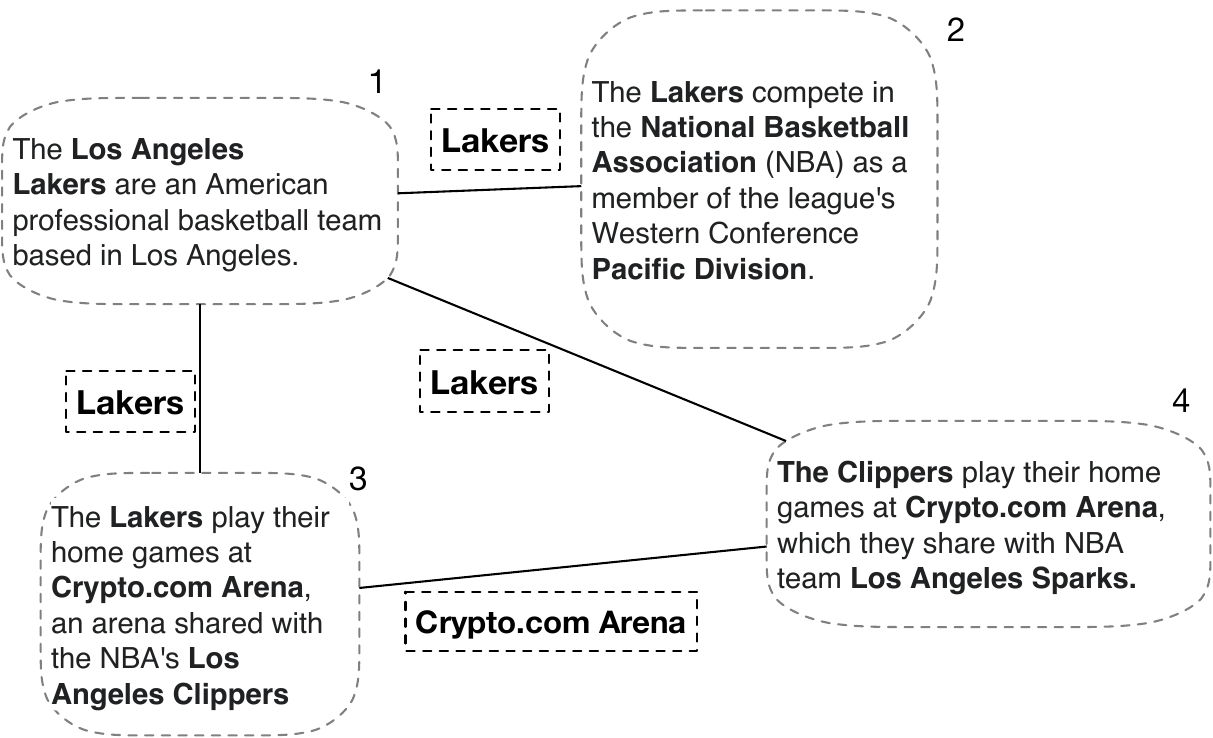}
    \caption{\textbf{Illustration of the Sentence graph.} In the sentence graph, nodes correspond to sentences and edges represent the coreference of entities across sentences. Sentences 1, 2 and 3 shares the entity \textit{Lakers} while sentence 4 shares the entity \textit{Crypto.com Arena} with sentence 3.}
    \label{fig:s_graph}
\end{figure}

As aforementioned, we construct the sentence graph to capture the semantic overlap of the factual sentences in the raw text. A proportion of the sentence graph is visualized in Figure~\ref{fig:s_graph}. Upon building the sentence graph, we aim at extracting the minimal sentence set that covers the most semantics in the whole graph. Now the question becomes how we can leverage the high-order co-reference relationship to reduce the size of the training data. To dive deep into this question, we start by making the following assumption:
\begin{assumption} \label{assumption:graph}
Suppose two sentences in a sentence set $S$, $\{s_e, s_e^\prime\} \subset S$, mention the same entity $e$.
The quality of a QA model $M_S$ trained by a sentence set $S$ will be similar to the quality of the other model $M_{S^\prime}$ trained by the set $S^\prime = S - \{s_e\}$ because $s_e^\prime\in S^\prime$ still cover the similar topics and knowledge in $s_e$.
\end{assumption}

Based on Assumption~\ref{assumption:graph}, an intuitive idea of leveraging the sentence graph to effectively reduce the size of the training data is to find a minimal set of sentence nodes that can cover the whole sentence graph without losing the quality of the model.
In other words, the challenge can be reduced to finding the \emph{minimal dominating set}~\cite{allan1978domination} of the sentence graph.

\subsection{Minimal Dominating Set Approximation}

Unfortunately, finding the minimal dominating set is an NP-Complete problem~\cite{hedetniemi1991bibliography}, so it is extremely time-consuming to obtain the optimal minimal dominating set as training data.
Hence, an efficient approximation approach to derive a decent dominating set with few enough sentences is essential.
To address this challenge, we leverage a greedy algorithm as shown in Algorithm~\ref{alg:dominantingset} by iteratively choosing the node that can cover the most uncovered nodes.

\begin{algorithm}
\caption{ApproximateDominantingSet}\label{alg:dominantingset}
\begin{algorithmic}
\State $S \gets \emptyset$
\State Let $H$ be a priority queue
\State Add all nodes in $H$ with their node degrees
\While{$H$ is not empty}
    \State $v \gets H.\text{pop\_max}()$
    \State $S \gets S \bigcup \{v\}$
    \State Remove $v$ and its neighbors in $E$ from $H$
    \State Update degrees of the remaining nodes in $H$
\EndWhile\\
\Return $S$
\end{algorithmic}
\end{algorithm}


\noindent \textbf{Complexity Analysis.}
Here we analyze the complexity of Algorithm~\ref{alg:dominantingset}.
Suppose $V$ and $E$ are the numbers of nodes and edges.
For time complexity, the algorithm first spends $O(V\log V)$ time to establish the max heap. For each iteration, taking the node with the highest degree costs $O(1)$ with the priority queue. In total, we need to update the priority queue $O(E)$ times, where each update costs $O(\log V)$ time.
Hence, the total time complexity is $O(E\log V)$.
For space complexity, the additional space complexity is only $O(V)$ to record the current set of uncovered nodes and the max heap.

\noindent \textbf{Theoretical Analysis.}
We also conduct some theoretical analysis on Algorithm~\ref{alg:dominantingset}.
According to Theorem~\ref{thm:approx}, the quality of dominating set derived by Algorithm~\ref{alg:dominantingset} is guaranteed. 
\begin{theorem}\label{thm:approx}
Algorithm~\ref{alg:dominantingset} computes an $(\ln \Delta + 2)$-approximation of the optimal dominanting set. In other words, 
for the computed dominating set $S$ and an optimal dominating set $S^*$, we have
$$\frac{|S|}{|S^*|} \leq \ln \Delta + 2,$$
where $\Delta=\max_v d(v)$ is the maximal degree of $G$.
\end{theorem}


\begin{proof}
Here we prove the theorem in an amortized way. 
Suppose each iteration costs 1 (i.e., contributing to the cardinality of the final dominating set).
Instead of letting the selected node takes all the cost, we amortize and distribute the cost among all newly covered nodes.

Assume $S^\prime$ is an optimal dominating set. 
By the definition of dominating set, we can assign each node in $V$ to exactly one neighboring node in $S^\prime$ so that the graph can be decomposed into several stars, where the center is a dominating node and non-dominating nodes are leaves.

Consider a certain star with a center $v^\prime \in S^\prime$ while choosing a node $u$ in Algorithm~\ref{alg:dominantingset}.
By the greedy condition and the optimality of $v^\prime$, after cost distribution, the charged cost of $u$ would be at most $d(v^\prime)$.
Also, after removing $u$, the degree of $v^\prime$ will be reduced by 1.
Following this process to iteratively select dominating nodes, the total amortized cost would be at most:

\begin{align*}
     \frac{1}{d(v^\prime) + 1}+\frac{1}{d(v^\prime) } +\cdots +\frac{1}{1} &= H(d(v^*)+1) \\
     &\leq  H(\Delta +1) \\
     & < \ln \Delta  + 2,
\end{align*}
where $\Delta$ is the maximal degree of the graph; $H(n)=\sum_{i-1}^n 1/i$.

\end{proof}

\subsection{Question Generation}
Our approach considers two question styles, including (1) generic cloze-style questions, wherein the answer is substituted by the token ``[MASK]", and (2) a templated question format termed "Wh+B+A+?" as well as its diverse ordering variations, as depicted in Figure \ref{fig:aug_data}. 
Given a retrieved sentence structured as \texttt{{[}Fragment\ A{]}\ {[}Answer{]}\ {[}Fragment\ B{]}}, the template "Wh + B + A +?" replaces the \emph{Answer} with a component Wh (for instance, what, who or where). This component is determined by the entity type of the \emph{Answer} and is placed at the beginning of the question. It is then followed by \texttt{Fragment\ B} and \texttt{Fragment\ A}. The selection of the wh-component involves sampling a bi-gram based on the likelihood of that particular bi-gram being connected with the named entity type of the answer. This likelihood is calculated from the named entity and questions bigram starters found in the SQuAD dataset. This information, while not leveraging the complete context-question-answer framework, can be considered as prior knowledge that does not disrupt the wholeness of our unsupervised methodology. It is also important to note that the choice of wh-component does not have a substantial impact on the results. Although we experimented with clause-based templates for this template-driven approach, we did not observe any significant differences in performance.
\begin{figure}[t]
    \centering
    \includegraphics[width=.9\linewidth]{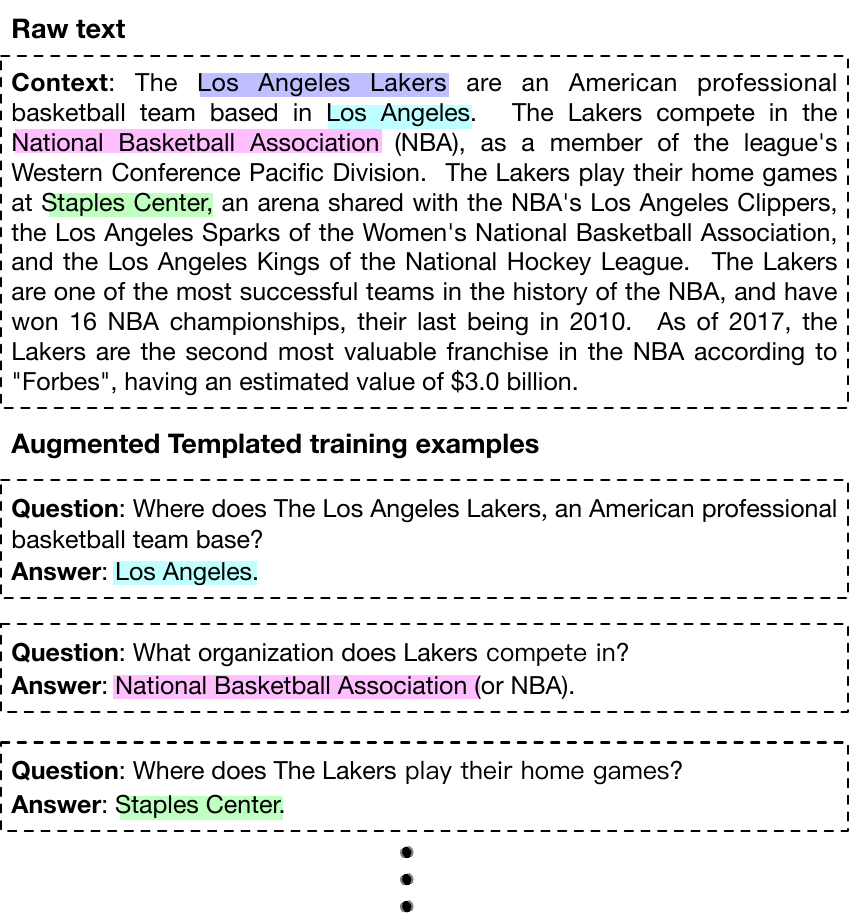}
    \caption{\textbf{Examples of generated questions.} When \ours runs into an $entity$ in the raw text during the question generation phase, it turns the factual sentence into a QA pair of $(question, entity)$, with the question type depending on the entity type.}
    \label{fig:aug_data}
\end{figure}
\subsection{Prompt-style Data Augmentation}
We extend the recent progress in prompt tuning to create augmented data for \ours. Specifically, we have formulated a template to enable QA input, designated as $x^{ori}$. The template is constructed as follows:
\begin{gather*}
    x_{q}= \emph{Question}: \mathbf{q} \\
    x_{a}= \emph{Answer}: {\rm <}\emph{mask}{\rm >} \\
    x_{c}= \emph{Context}: \mathbf{c} \\
    x^{ori}=\left[x_{q} \oplus x_{a} \oplus x_{c}\right]
\end{gather*}
Here, we formulate the labels $y$ as:   
\begin{gather*}
    y_{a}= \emph{Answer}: \mathbf{a}, \\
    y=\left[x_{q} \oplus y_{a} \oplus x_{c}\right],
\end{gather*}
where $\mathbf{q}$, $\mathbf{a}$, and $\mathbf{c}$ represent the query text, response text, and background context respectively, and $\oplus$ symbolizes string concatenation.


In the augmented QA data samples, we apply the masking to the chosen entity in $x_{c}$ to construct the context text for the augmented data $x_{c}^{\emph{aug}}$, along with the mask token in $x_{a}$. The specifics of an augmented data sample $(x^{\emph{aug}}, y^{\emph{aug}})$ are depicted in Figure~\ref{fig:aug_data}. Let the set of all training samples from original QA datasets and augmented QA pairs be denoted by $(X^{\emph{ori}}, Y^{\emph{ori}})$ and $(X^{\emph{aug}}, Y^{\emph{aug}})$ respectively. Thus, our entire training set $(X^{\emph{train}}, Y^{\emph{train}})$ comprises of both $(X^{\emph{ori}}, Y^{\emph{ori}})$ and $(X^{\emph{aug}}, Y^{\emph{aug}})$.

\subsection{Training}
One of the key benefits of harmonizing the augmented and original data lies in the ability of the model to effectively process both data types without any significant loss. Concisely, \ours derives a prediction utilizing an encoder-decoder model as 
\begin{equation}
    y^{\emph{pred}} = {\rm \emph{decoder}}_{\theta_{\emph{D}}}({\rm \emph{encoder}}_{\theta_{\emph{E}}}(x)), 
\end{equation}
where $\theta_{\emph{E}}$ and $\theta_{\emph{D}}$ represent learnable parameters, and $x \in X^{train}$ can be either an original or an augmented training sample.

The training objective of our system aims to maximize the log-likelihood of the text in the reference answer, denoted by $y \in Y^{train}$. The loss functions concerning the original samples and the augmented samples are expressed in the following equations:
\begin{align*}
L^{\emph{ori}}(\theta)\!&=\!\!\! 
\sum_{\left(x, y\right) \in\left(X^{\emph{ori}}, Y^{\emph{ori}}\right)} \!\!\log \left(\prod_{i=1}^{n} P\left(y_{i} \mid y_{<i}, x; \theta\right)\right) \\
L^{aug}(\theta)\!&= \!\!\!
\sum_{\left(x, y\right) \in\left(X^{\emph{aug}}, Y^{\emph{aug}}\right)} \!\!\log \left(\prod_{i=1}^{n} P\left(y_{i} \mid y_{<i}, x; \theta\right)\right)
\end{align*}
%
%
where $\theta = \{\theta_{D}, \theta_{E}\}$.
The overall loss function is the weighted average of two losses:
\begin{equation}
    L(\theta) = L^{\emph{ori}}(\theta) + \lambda L^{\emph{aug}}(\theta).
\end{equation}

We consider $\lambda >0$ to be a hyperparameter that establishes a balance between the few-shot QA training samples and the augmented QA samples.

\section{Experiments}
\label{sec:exp}

\begin{table*}[]
\centering
\resizebox{\linewidth}{!}{
\begin{tabular}{ccccccccc}
\toprule
\textbf{\# examples} & \textbf{SQuAD} & \textbf{TriviaQA} & \textbf{NQ}   & \textbf{NewsQA} & \textbf{SearchQA} & \textbf{HotpotQA} & \textbf{BioASQ} & \textbf{TextbookQA} \\ \midrule
\# nodes                                       & 104,160                                  & 123,183                                     & 418,049                               & 356,408                                   & 25,413                                      & 417,895                                     & 60,080                           & 30,723                               \\
\# edges                                       & 20,310,486                               & 36,716,957                                  & 408,935,741                           & 339,619,544                               & 13,425,062                                  & 766,206,565                                 & 6,821,645                        & 3,150,557                            \\
\# dominating set                              & 8,260                                    & 11,099                                      & 30,452                                & 24,015                                    & 1,518                                       & 34,830                                      & 4,480                            & 1,116                                \\
\textbf{\# training samples}                   & \textbf{17,409}                          & \textbf{24,091}                             & \textbf{48,213}                       & \textbf{32,391}                           & \textbf{4,509}                              & \textbf{116,385}                            & \textbf{6,884}                   & \textbf{1,505}                       \\ \bottomrule
\end{tabular}}
\caption{\textbf{Number of augmented training examples per dataset.} We construct one training example per entity extracted from the raw text of each QA dataset and use the \ours to produce augmented QA data.}
\label{tab:aug_data}
\end{table*}

\begin{table*}[t!]
\centering
\resizebox{\linewidth}{!}{
\begin{tabular}{l|cccccccc|c}
\toprule
\multicolumn{1}{c}{\textbf{Model}} & \multicolumn{1}{|c}{\textbf{SQuAD}}    & \multicolumn{1}{c}{\textbf{TriviaQA}} & \multicolumn{1}{c}{\textbf{NQ}}       & \multicolumn{1}{c}{\textbf{NewsQA}}   & \multicolumn{1}{c}{\textbf{SearchQA}} & \multicolumn{1}{c}{\textbf{HotpotQA}} & \multicolumn{1}{c}{\textbf{BioASQ}}   & \multicolumn{1}{c}{\textbf{TextbookQA}}   &\multicolumn{1}{|c}{\textbf{Average}}   \\ \toprule
\multicolumn{9}{l}{16 Examples}           \\ \bottomrule
RoBERTa &7.7±4.3 &7.5±4.4 &17.3±3.3 &1.4±0.8 &6.9±2.7 &10.5±2.5 &16.7±7.1 &3.3±2.1 &9.0±3.4 \\
SpanBERT &18.2±6.7 &11.6±2.1 &19.6±3.0 &7.6±4.1 &13.3±6.0 &12.5±5.5 &15.9±4.4 &7.5±2.9 &13.3±4.3 \\
PMR & 60.3±4.0 & \textbf{56.2±3.1} & 43.6±1.7 & 30.1±3.7 & \textbf{58.2±5.0} & 46.1±4.7 & 54.2±3.4 & 31.0±1.8   & 47.5±3.4 \\ \midrule
Splinter &54.6±6.4 &18.9±4.1 &27.4±4.6 &20.8±2.7 &26.3±3.9 &24.0±5.0 &28.2±4.9 &19.4±4.6 &27.4±4.5 \\
Splinter w/ \ours &\textbf{58.9±3.6} &\textbf{35.7±1.9} &\textbf{37.6±2.8} &\textbf{31.9±1.8} &\textbf{35.2±1.6} &\textbf{34.0±6.3} &\textbf{38.7±3.6} &\textbf{37.0±5.1} &\textbf{36.1±3.3} \\ \midrule
FewshotQA &72.5±3.7 &47.1±7.6 &57.3±3.2 &44.9±4.5 &54.3±5.9 &\textbf{59.7±2.2} &\textbf{62.7±4.4} &33.1±3.2 &53.9±4.3 \\
FewshotQA w/ \ours &\textbf{73.6±3.3} &50.9±4.6 &\textbf{58.5±1.9} &\textbf{46.5±1.8} &55.4±2.7 &57.1±2.9 &57.2±2.3 &\textbf{42.2±4.1} &\textbf{55.2±2.9} \\    \toprule
\multicolumn{9}{l}{32 Examples}           \\ \bottomrule
RoBERTa &18.2±5.1 &10.5±1.8 &22.9±0.7 &3.2±1.7 &13.5±1.8 &10.4±1.9 &23.3±6.6 &4.3±0.9 &13.3±2.6 \\
SpanBERT &25.8±7.7 &15.1±6.4 &25.1±1.6 &7.2±4.6 &14.6±8.5 &13.2±3.5 &25.1±3.3 &7.6±2.3 &16.7±4.7 \\
PMR & 70.0±3.2 & \textbf{66.3±2.5} & 48.5±3.5 & 36.6±2.1 & \textbf{64.8±2.2} & 52.9±2.5 & 62.9±2.4 & 36.4±3.2   & 54.8±2.7 \\ \midrule
Splinter &59.2±2.1 &28.9±3.1 &33.6±2.4 &27.5±3.2 &34.8±1.8 &34.7±3.9 &36.5±3.2 &27.6±4.3 &35.3±3.0 \\
Splinter w/ \ours &\textbf{64.6±1.5} &\textbf{35.6±2.1} &\textbf{42.8±1.3} &\textbf{33.0±1.2} &\textbf{39.2±3.4} &\textbf{41.4±3.1} &\textbf{49.2±3.2} &\textbf{38.2±2.5} &\textbf{43.0±2.3} \\ \midrule
FewshotQA &73.8±2.2 &56.7±5.9 &\textbf{60.6±2.4} &50.0±2.8 &61.4±3.6 &61.6±1.5 &\textbf{66.9±4.7} &41.7±4.2 &59.1±3.4 \\
FewshotQA w/ \ours &\textbf{78.0±1.1} &53.5±4.0 &59.3±1.0 &\textbf{51.8±1.8} &60.3±2.6 &\textbf{61.6±3.1} &63.6±2.9 &\textbf{46.5±2.0} &\textbf{59.3±2.3} \\ \toprule
\multicolumn{9}{l}{64 Examples}           \\ \bottomrule
RoBERTa &28.4±1.7 &12.5±1.4 &24.2±1.0 &4.6±2.8 &19.8±2.4 &15.0±3.9 &34.0±1.8 &5.4±1.1 &18.0±2.0 \\
SpanBERT &45.8±3.3 &15.9±6.4 &29.7±1.5 &12.5±4.3 &18.0±4.6 &23.3±1.1 &35.3±3.1 &13.0±6.9 &24.2±3.9 \\
PMR & 71.2±2.8 & \textbf{67.1±1.8} & 51.2±3.1 & 43.2±1.8 & 66.2±1.8 & 56.3±2.0 & 68.2±1.6 & 41.8±2.3   & 58.1±2.2 \\ \midrule
Splinter &65.2±1.4 &35.5±3.7 &38.2±2.3 &\textbf{37.4±1.2} &39.8±3.6 &45.4±2.3 &49.5±3.6 &35.9±3.1 &43.4±2.7 \\
Splinter w/ \ours &\textbf{68.6±1.8} &\textbf{35.4±2.9} &\textbf{45.9±1.3} &36.1±1.7 &\textbf{44.3±3.1} &\textbf{48.6±2.3} &\textbf{59.4±2.4} &\textbf{42.6±1.6} &\textbf{47.6±2.1} \\ \midrule
FewshotQA &77.9±2.1 &57.9±4.4 &\textbf{60.9±2.5} &53.7±1.1 &65.4±2.4 &\textbf{63.1±2.2} &\textbf{73.2±3.1} &44.8±1.8 &62.1±2.5 \\
FewshotQA w/ \ours &\textbf{79.2±1.0} &55.3±3.2 &59.7±1.3 &\textbf{54.2±1.0} &\textbf{67.1±1.0} &61.1±3.0 &72.4±2.5 &\textbf{48.7±2.4} &\textbf{62.5±1.9} \\ \toprule
\multicolumn{9}{l}{128 Examples}           \\ \bottomrule
RoBERTa &43.0±7.1 &19.1±2.9 &30.1±1.9 &16.7±3.8 &27.8±2.5 &27.3±3.9 &46.1±1.4 &8.2±1.1 &27.3±3.1 \\
SpanBERT &55.8±3.7 &26.3±2.1 &36.0±1.9 &29.5±7.3 &26.3±4.3 &36.6±3.4 &52.2±3.2 &20.9±5.1 &35.4±3.9 \\
PMR & 79.8±1.8 & \textbf{68.6±1.4} & 57.4±2.6 & 52.3±1.4 & \textbf{68.5±1.8} & \textbf{65.9±1.0} & 76.8±2.1 & 45.1±1.2   & \textbf{64.3±1.7} \\ \midrule
Splinter &\textbf{72.7±1.0} &\textbf{44.7±3.9} &46.3±0.8 &\textbf{43.5±1.3} &47.2±3.5 &\textbf{54.7±1.4} &63.2±4.1 &42.6±2.5 &51.9±2.3 \\
Splinter w/ \ours &70.2±2.8 &45.4±1.3 &\textbf{51.2±1.3} &40.2±1.6 &\textbf{48.5±2.1} &54.5±2.2 &\textbf{67.8±1.6} &\textbf{44.2±2.1} &\textbf{52.8±1.9} \\ \midrule
FewshotQA &78.8±2.7 &55.2±1.8 &63.3±1.6 &56.8±1.1 &67.0±1.8 &64.9±1.8 &77.2±1.5 &46.2±5.9 &63.7±2.3 \\
FewshotQA w/ \ours &\textbf{80.5±1.4} &52.9±3.9 &\textbf{64.2±1.4} &\textbf{56.9±1.0} &68.1±1.9 &61.7±1.4 &\textbf{77.8±1.2} &\textbf{52.5±3.7} &\textbf{64.3±2.0} \\ \bottomrule
\end{tabular}}
\caption{ \textbf{Overall performance} in F1 scores across all datasets when the numbers of training examples are 16, 32, 64, and 128. NQ stands for Natural Questions. RoBERTa, SpanBERT, Splinter and Splinter w/ \ours have 110M parameters. PMR, FewshotQA and FewshotQA w/ \ours have parameters of size 406M. Comparisons with more baselines are in Section~\ref{sec:qasar} and Appendix~\ref{sec:mqa-qg}.}
\label{tab:overall_perf}
\end{table*}

\subsection{Experimental Setup}
\noindent \textbf{Datasets.}
Following Splinter~\cite{ram2021few} and FewshotQA~\cite{chada2021fewshotqa}, we sample subsets from the MRQA 2019 shared task~\cite{fisch2019mrqa} for our few-shot experiments.
Taking a closer look, there are in total eight widely used benchmark QA datasets in MRQA: SQuAD~\cite{rajpurkar2016squad}, NewsQA~\cite{trischler2017newsqa}, TriviaQA~\cite{joshi2017triviaqa}, SearchQA~\cite{dunn2017searchqa}, HotpotQA~\cite{yang2018hotpotqa}, Natural Questions~\cite{kwiatkowski2019natural}, BioASQ~\cite{tsatsaronis2015overview}, and TextbookQA~\cite{kembhavi2017you}. 
Following Splinter~\cite{ram2021few}, smaller training datasets are sampled in a logarithmic manner from the original full datasets, resulting in few-shot datasets with  16, 32, 64, and 128 training examples.

\noindent \textbf{Comparative Baselines.} We evaluate the performance of \ours against four competitive few-shot QA methods, including  \textbf{RoBERTa~}\cite{liu2019roberta}, \textbf{SpanBERT}~\cite{joshi2020spanbert}, \textbf{Splinter}~\cite{ram2021few}, \textbf{ FewshotQA}~\cite{chada2021fewshotqa}, and \textbf{PMR}~\cite{xu2023cloze}. Details of these baselines, raw text data source, and evaluation metric are in Appendix~\ref{sec:baseline_details},~\ref{sec:raw_data} and ~\ref{sec:eval_metric}, correspondingly.

\subsection{Implementation Details}
\label{sec:imple_det}
For all the models, we use the same hyperparameters during training for a fair comparison. Specifically, the models are optimized by Adam \cite{kingma2014adam} with bias corrections. The learning rate is $2 \times 10^{-5}$ without learning rate scheduling. The training batch size is set to $2$. The maximum sequence length of sequence generation is $100$ for FewshotQA and \ours. We train all the models compared for $25$ epochs. The reported results are given by the best-performing checkpoint in the development sets. For \ours, we perform a grid search for the loss weight $\lambda$ in the space $\{0.01, 0.05, 0.1, 0.5, 1.0, 10.0\}$. 
All experiments are run on NVIDIA Tesla A100-SXM4 Tensor Core GPUs with 40GB memory.

\subsection{Performance Comparison}
Table~\ref{tab:overall_perf} presents the few-shot QA performance comparison of various models across all benchmarks when provided with 16, 32, 64, and 128 training examples. BART-large serves as the backbone pre-trained language model (PLM) for FewshotQA.

The experiment was repeated five times, each with a different random seed, and we report the average and standard deviation of the results for each method. As a general observation, PMR, Splinter and FewshotQA with \ours excel over other compared methods by a respectable margin in most cases. On average, models with \ours yield better results with consistently lower variances (the rightmost column). The only exception is the 128 examples, where \ours and PMR ended in a draw. Note that FewshotQA with \ours performs better in fewer-shot cases because BART is pretrained on general domain plain texts, so \ours can apply its broad knowledge and rapidly adapt to the specifics of the QA task with just a few examples. PMR gradually catches up with more few-shot examples because its specialized training allows it to learn more efficiently from and utilize the additional examples, scaling its performance in a way that is directly relevant to the task. There are several cases in which performance degrades when using \ours. This is probably because the augmented data samples outweigh the original fine-tuning data samples for these datasets, directing the pretrained model towards the distribution of the augmented data which is slightly shifted from the distributions of the fine-tuning and test data after all. More notably, \ours exhibits less variance in results compared to FewshotQA in most cases, particularly when there are fewer training examples available.

\begin{table}[h]\centering
\scriptsize
\renewcommand{\arraystretch}{.7}
\begin{tabular}{lccc}\toprule
\textbf{Model} &\textbf{SQuAD} &\textbf{TextbookQA} \\ \midrule
\multicolumn{3}{l}{16 Examples} \\ \midrule
FewshotQA w/ \ours-random &72.0±3.5 &39.2±4.8 \\
FewshotQA w/ \ours &\textbf{73.6±3.3} &\textbf{42.2±4.1} \\ \midrule
\multicolumn{3}{l}{32 Examples} \\ \midrule
FewshotQA w/ \ours-random &75.9±1.8 &43.3±2.2 \\
FewshotQA w/ \ours &\textbf{78.0±1.1} &\textbf{46.5±2.0} \\ \midrule
\multicolumn{3}{l}{64 Examples} \\ \midrule
FewshotQA w/ \ours-random &78.6±1.3 &46.2±2.2 \\
FewshotQA w/ \ours &\textbf{79.2±1.0} &\textbf{48.7±2.4} \\ \midrule
\multicolumn{3}{l}{128 Examples} \\ \midrule
FewshotQA w/ \ours-random &79.9±1.4 &49.5±3.5 \\
FewshotQA w/ \ours &\textbf{80.5±1.4} &\textbf{52.5±3.7} \\
\bottomrule
\end{tabular}
\caption{\textbf{Ablation study.} Comparison between \ours and randomly selecting the same amount of sentences and generating training samples. }
\label{tab:ablation}
\end{table}

In digging deeper into specific models, both Splinter and FewshotQA enhanced by \ours consistently outperform their original model in terms of higher F1 scores with generally lower variances. On SQuAD, NQ, BioASQ, and TextbookQA, the performance improvements over the top baseline are relatively more substantial. Our hypothesis is that the factual statements are more concentrated in a small number of sentences, thus \ours can more effectively extract the most informative data for fine-tuning. Consequently, the influence from the is adequate to impact the primary QA task. We also observe that with the decrease in the number of few-shot QA training examples, \ours demonstrate more improvement. This is also expected since \ours essentially introduces external prior knowledge that is not present in the few-shot training examples. When the models see more actual training examples that are with the same distribution as the test set, the external knowledge helps less and even becomes noise in the extreme case. Finally, we also observe a greater improvement brought about by \ours to Splinter than to FewshotQA. This is because Splinter has a smaller model size; therefore, it naturally acquires less knowledge during the pre-train stage. Adding external knowledge to it in the form of QA benefits even more than bigger models, such as FewshotQA.

\begin{figure*}[t!]
    \centering
    \includegraphics[width=.9\linewidth]{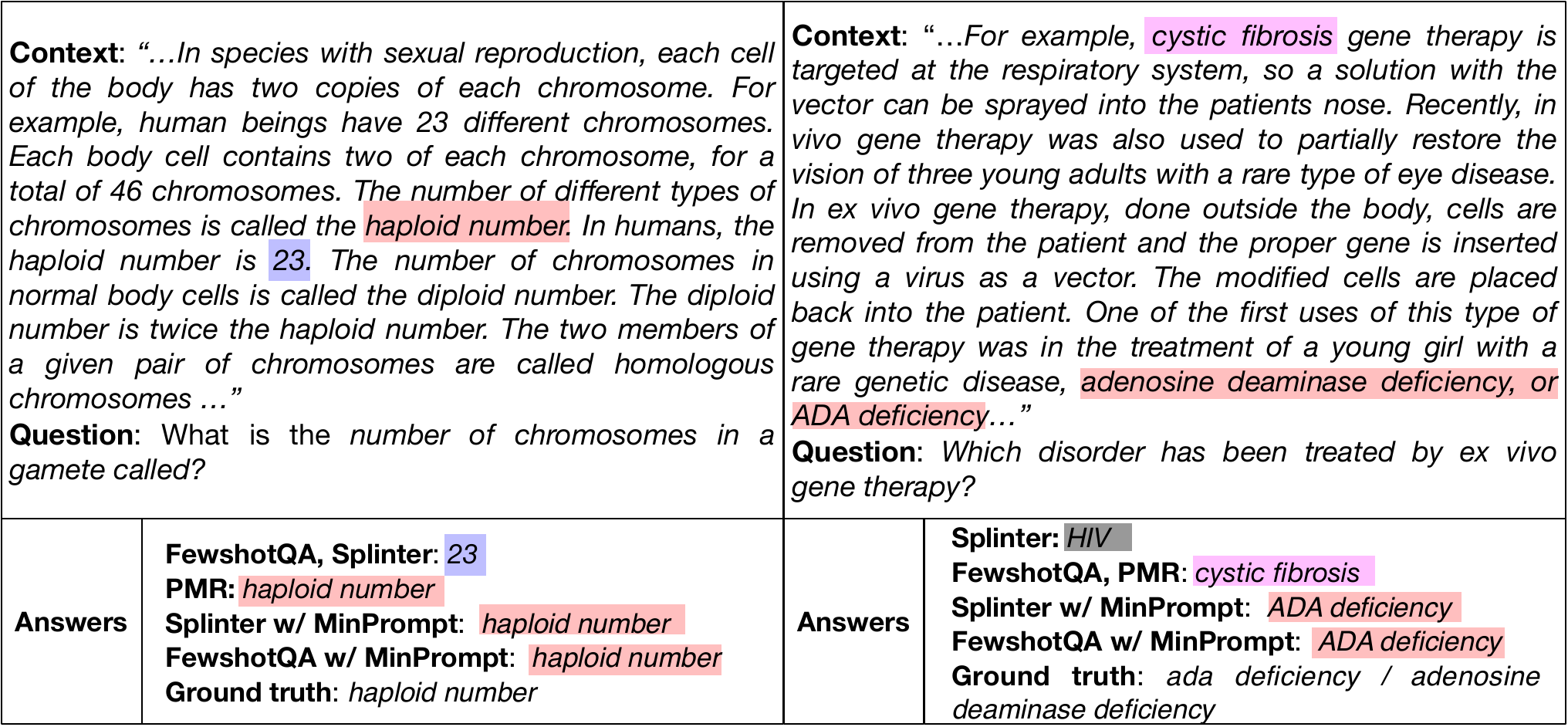}
    \caption{\textbf{Case study.} In both cases, \ours successfully generates the correct answer, whereas baselines without entity masking can not accurately recover the entity-level details. }
    \label{fig:case_study}
\end{figure*}
\subsection{Effect of Deriving the Dominating Set}
\label{sec:ablation}
To validate the necessity of deriving the dominating set of the sentence graph to keep the most informative factual sentences in the raw text, we further conduct an ablation study. We construct a variant of \ours called \ours-random where we randomly sample the same number of sentences as shown in Table~\ref{tab:aug_data} for each dataset, and then generate training samples out of these randomly sampled factual sentences. We run \ours-random and report the results on SQuAD and TextbookQA in Table~\ref{tab:ablation}. When comparing the two models, we can observe that \ours consistently perform better than \ours-random. We also observe this pattern on all the other datasets. This observation empirically validates that the dominating set derivation process indeed provides factual sentences that preserve as much information as possible about the crucial entities in the raw text.

\subsection{Case Study}
Further exploration of two specific test cases from the TextbookQA test set provides insightful results, as depicted in Figure~\ref{fig:case_study}. In the left case, both FewshotQA and Splinter without \ours yield the incorrect response, \textit{23}. Despite its semantic relevance to the accurate answer, \textit{haploid number}, the response goes overly detailed, since the value \textit{23} is specific only to human beings. This case underlines the advantage of \ours's full model, equipped with a sentence construction module anchored by entities, in deriving detailed answer text at the entity level, over FewshotQA and Splinter. In the right case, both FewshotQA and Splinter with \ours successfully identify the correct answer, whereas Splinter supplies an incorrect answer, \textit{HIV}, not even present in the context. Meanwhile, FewshotQA and PMR produced another treatment instead of what the question asks (a disorder), indicating that the question generation module of \ours improved the models' ability to deal with various kinds of questions. This comparison effectively highlights the utility of the sentence graph in forging higher-order entity interconnections within the same context. Although the baselines provide a contextually relevant response, they do not adequately address the question. The two cases substantiate the indispensable role of the sentence graph construction module and the question generation module in \ours, fortifying its capacity to delve into complex question and context semantics.

\begin{table}[!t]
    \centering
    \resizebox{\linewidth}{!}{
    \begin{tabular}{c|cccc}
         \textbf{Model} & \textbf{NQ} & \textbf{NewsQA} & \textbf{BioASQ} & \textbf{TextbookQA}  \\ \hline
        \textbf{Qasar} & 59.76 & 56.63 & 63.70 & 47.02 \\
         \textbf{Splinter w/ MinPrompt} & 51.17 & 40.22 & 67.80 & 44.24 \\
\textbf{FewshotQA w/ MinPrompt}  & \bf 64.17 & \bf 56.84 & \bf 77.84 & \bf 52.53 \\ \hline
    \end{tabular}}
    \caption{Performance of MinPrompt with 128 examples against the unsupervised domain adation method.}
    \label{tab:against_unsupervised_adaption}
\end{table}

\subsection{Comparisons against Unsupervised Domain Adaption}
\label{sec:qasar}
In addition to the few-shot approach, some studies apply unsupervised domain adpation to tackle the limitation of training data \cite{assem2021qasar}.
As an additional study, we compare with Qasar~\cite{assem2021qasar} Qasar, we focus on four overlapping datasets (i.e., NQ, NewsQA, BioASQ, and TextbookQA) between their paper and our studies as shown in Table~\ref{tab:against_unsupervised_adaption}. We can observe that FewshotQA w/ MinPrompt outperforms Qasar across four datasets from 0.4\% to 22.2\%. We also would like to emphasize that Qasar uses fine-tuning training samples ranging from 142 to 4,185 while MinPrompt using only 16 to 128 fine-tuning examples surpasses Qasar with certain disadvantages in the limited amount of fine-tuning data.

\section{Conclusion}

In this paper, we present \ours, a robust data augmentation framework that leverages a graph-based algorithm and unsupervised question generation to extract minimally meaningful QA training samples from raw text. Our contributions reside in the application of minimal data augmentation, enhancing computational efficiency and model performance while mitigating overfitting. Through extensive experiments, our model consistently outperformed competitive methods in public benchmarks, demonstrating its effectiveness.

\section*{Acknowledgements}
We thank anonymous reviewers for their valuable and insightful feedback. Research was supported in part by NIH U24DK097771, U54HG012517, NSF 1829071, 2106859, 2119643, 2200274,  2202693, and 2312501, DARPA HR00112490370, and Optum Lab. 

\section*{Limitations}
While \ours is capable of achieving comparative or better performance over existing studies, it still has some limitations as follows:
First, \ours integrates the trained NER model as part of the pipeline, so the performance of the SpaCy NER model greatly affects the overall performance of \ours.
Second, \ours uses all shared entities to construct the sentence graph. However, some entities might be more crucial than others for the downstream QA task. As a result, treating the entities differently might lead to a different result.
Lastly, the template utilized for prompt-tuning in this study still relies on manual design. Our approach is influenced by previous research that has been shown to be effective. Nevertheless, it would be intriguing to explore the development of automated methods for constructing superior prompt-tuning templates.

\section*{Ethics Statement}
This paper presents work that aims to advance the field of Natural Language Processing, specifically Large Language Models. There are potential societal consequences of our work associated with LLMs, such as AI safety and reliability. Beyond LLMs, we feel no other consequences must be highlighted here.

\bibliography{anthology,custom}

\newpage
\appendix

\section{Baseline Details}
\label{sec:baseline_details}
\begin{itemize}[leftmargin=*]
\item \textbf{RoBERTa \cite{liu2019roberta}} is a robustly optimized BERT-based PLM. It improves BERT by techniques such as training the model for a longer time, with larger batches and getting rid of the next sentence prediction task. It is known to demonstrate substantially better performance on a variety of natural language understanding tasks over BERT, including QA.

\item \textbf{SpanBERT \cite{joshi2020spanbert}} is another variant of BERT that emphasizes the encoding of spans instead of tokens. It is pretrained on two tasks: (1) masked language modeling, which is the same as BERT, and (2) span boundary prediction, which pulls the representations of the span boundary into a direction where the entire content of the masked span can be predicted correctly. SpanBERT achieves substantially better performance on span selection tasks in particular.

\item \textbf{Splinter \cite{ram2021few}} is a pretraining framework dedicated to the extractive QA task based on SpanBERT. It is pretrained by the recurring span selection task, which masks all but one instance of each recurring span and asks the model to select the correct span for each masked position.

\item \textbf{FewshotQA \cite{chada2021fewshotqa}} is the first QA-dedicated fine-tuning framework that takes advantage of pre-trained encoder-decoder models such as BART~\cite{lewis2020bart} and T5~\cite{raffel2020exploring}. In FewshotQA, the input is constructed as a concatenation of the question, a mask token as the placeholder for the answer span, and a context. Given this input, the model is fine-tuned using the same objective as its pretraining objective.

\item \textbf{PMR \cite{xu2023cloze}} constructs general-purpose machine reading comprehension training data by using Wikipedia hyperlinks and designed a Wiki Anchor Extraction task to guide the MRC-style pretraining.
\end{itemize}

\section{QA data acquisition}
\label{sec:raw_data}
The first step in our framework is to retrieve the raw text corpus as the super set from which all our prompt dataset comes. For pretraining, text corpus from general domains such as Wikipedia is commonly used. On the contrary, since we focus on the fine-tuning stage, we use domain-specific text as a starting point. Following Splinter~\cite{ram2021few} and FewshotQA~\cite{chada2021fewshotqa}, we take MRQA~\cite{fisch2019mrqa} as a benchmark to test the performance of all the comparative methods.

\begin{table*}[!t]
\renewcommand{\arraystretch}{1}
\small
    \centering
    \begin{tabular}{c|cccccccc} \hline
         \textbf{Model} & \textbf{SQuAD} & \textbf{TriviaQA} & \textbf{NQ} & \textbf{NewsQA} & \textbf{SearchQA} & \textbf{HotpotQA} & \textbf{BioASQ} & \textbf{TextbookQA} \\ \hline
         \textbf{MQA-QG} & 54.38 & 32.28 & 37.36 & 25.12 & 31.35 & 33.89 & 36.39 & 29.71 \\
         \textbf{MinPrompt} & 58.91 & 35.67 & 37.64 & 31.88 & 35.17 & 34.03 & 38.68 & 36.98 \\ \hline
    \end{tabular}
    \caption{Performance comparisons against MQA-QG.}
    \label{tab:MQA-QG}
\end{table*}

\section{Evaluation Metrics}
\label{sec:eval_metric}
Following previous studies~\cite{ram2021few,chada2021fewshotqa}, we use the F1 score as our evaluation metric. Specifically, for each sample in the test set, the predicted span and the ground truth answer are treated as bags of words, and F1 scores are applied to compute the overlap between these two sets. If there are multiple ground-truth answers to a particular question, we take the maximum of the corresponding F1 scores.

\section{Comparisons against MQA-QG}
\label{sec:mqa-qg}
Here we compare with the other few-shot data augmentation approach, MQA-QG~\cite{pan2021unsupervised}. For a fair comparison, we first run the released implementation of MQA-QG, apply their approach on Splinter, and then compare it with our method. The results of 16-shot experiments are as shown in Table~\ref{tab:MQA-QG}. We see consistent improvements derived by MinPrompt over MQA-QG, and a similar pattern is also observed in 32, 64, and 128-shot scenarios.

\section{Additional Discussions}

Here we list some additional discussions on our approach.

\subsection{Generalization Ability to Different Answer Types}
To different types of answers (e.g., why v.s. how and longanswers), we would like to mention that MinPrompt raises different types of questions based on the results of the entity typing. During this process, why / how questions would be raised once a conjunction (e.g., because) or an adverb (e.g., by) is recognized from the raw text. We agree that the why / how questions with longer answers might be less than some other types of questions like what / who / when ones in the augmented training samples, and it might cause generalization issues. An intuitive fix is to assign larger sample weights to the augmented samples with why / how questions or to repeat these samples multiple times to make different types of questions roughly be of the same number. However, the main focus of this paper is to demonstrate the idea that graph-based data selection can help the overall downstream performance, so we leave the detailed analysis for certain types of answers for future work.

\subsection{Potential Solution to Overfitting with Prompt-style Augmentation}

It could introduce an ovefit with prompt-style agumentation to the distribution of different quesetion formats as we observed in the experiments, especially for the cases with only few shot training samples. The distribution of different types of questions in the augmented data might be skewed, for example, the what / who / when questions might be more than the why / how questions. In this way, the what / who / when questions in the test set might get more precise answers than the why / how questions. The intuitive fix is to put larger sample weights to the augmented samples with why / how questions or to repeat these samples multiple times to make different types of questions roughly be of the same number.

\end{document}